\documentclass[runningheads]{llncs}
\usepackage[T1]{fontenc}
\usepackage{graphicx}
\usepackage{booktabs}
\usepackage[misc]{ifsym}
\usepackage{xcolor}

\usepackage{mwe}

\usepackage{multirow}%
\usepackage{amsmath,amssymb,amsfonts}%
\usepackage{mathrsfs}%
\usepackage{textcomp}%
\usepackage{manyfoot}%
\usepackage{booktabs}%
\usepackage{algorithm}%
\usepackage{algorithmic}%
\usepackage{listings}%
\usepackage[colorlinks=true]{hyperref}
\usepackage{rotating}
\usepackage{adjustbox}

\begin{document}

\title{Learning Partial Action Replacement in Offline MARL}


\author{Yue Jin\inst{1} \and
Giovanni Montana\inst{2,3} }


\institute{Faculty of Arts, Science \& Technology, 
University of Northampton, Waterside Campus, 
Northampton NN1 5PH, Northamptonshire, UK \email{esther.jin@northampton.ac.uk}
\and
Warwick Manufacturing Group, University of Warwick, Coventry, CV4 7AL, UK
\and 
Department of Statistics, University of Warwick, Coventry, CV4 7AL, UK \email{g.montana@warwick.ac.uk}}

\maketitle              

\begin{abstract}
Offline multi-agent reinforcement learning (MARL) faces a critical challenge: the joint action space grows exponentially with the number of agents, making dataset coverage exponentially sparse and out-of-distribution (OOD) joint actions unavoidable. Partial Action Replacement (PAR) mitigates this by anchoring a subset of agents to dataset actions, but existing approach relies on enumerating multiple subset configurations at high computational cost and cannot adapt to varying states. We introduce PLCQL, a framework that formulates PAR subset selection as a contextual bandit problem and learns a state-dependent PAR policy using Proximal Policy Optimisation with an uncertainty-weighted reward. This adaptive policy dynamically determines how many agents to replace at each update step, balancing policy improvement against conservative value estimation. We prove a value-error bound showing that the estimation error scales linearly with the expected number of deviating agents. Compared with the previous PAR-based method SPaCQL, PLCQL reduces the number of per-iteration Q-function evaluations from $n$ to $1$, significantly improving computational efficiency. Empirically, PLCQL achieves the highest normalised scores on 66\% of tasks across MPE, MaMuJoCo, and SMAC benchmarks, outperforming SPaCQL on 84\% of tasks while substantially reducing computational cost.
\end{abstract}


\section{Introduction}
Offline multi-agent reinforcement learning (MARL) aims to learn coordinated policies from fixed datasets without additional interaction with the environment. This setting is particularly attractive for real-world applications where online exploration is costly, unsafe, or impractical, such as robotics, autonomous systems, and large-scale resource management \cite{Prudencio2022AProblems,Levine2020,Mandlekar2021WhatManipulationb,Tesauro2006AAllocation}. However, offline MARL introduces a fundamental challenge that is significantly more severe than in the single-agent setting: the combinatorial explosion of the joint-action space. As the number of agents increases, the number of possible joint actions grows exponentially, while any fixed dataset can cover only a small fraction of these combinations. Consequently, learning algorithms must estimate the value of many out-of-distribution (OOD) joint actions. When combined with function approximators such as neural networks, value-based methods may assign arbitrarily large Q-values to these unseen actions, leading to severe overestimation and policies that exploit value errors rather than effective coordination \cite{Yang2021BelieveLearning,Shao2023CounterfactualLearning,Kumar2020,Kumar2019StabilizingReduction}.

One promising strategy for mitigating this problem is to constrain policy updates to remain close to the support of the dataset. Recent work introduces partial action replacement (PAR), which changes the actions of only a subset of agents while keeping the remaining actions fixed to those observed in the dataset \cite{Jin2025PartialMARLb}. By anchoring part of the joint action to data-supported actions, PAR reduces distributional shift and stabilizes offline policy learning. However, the effectiveness of PAR critically depends on which subset of agents is selected for action replacement. Existing methods, such as Soft Partial Conservative Q-Learning (SPaCQL), attempt to address this issue by enumerating multiple subset sizes and randomly sampling combinations for each size, followed by an uncertainty-based weighting scheme over the resulting target values \cite{Jin2025PartialMARLb}. Although this approach can improve robustness, it introduces substantial computational overhead because multiple PAR configurations must be evaluated at every update step. Moreover, the uncertainty-based weighting mechanism does not explicitly prioritize subsets that lead to stronger multi-agent coordination. As a result, existing PAR-based approaches rely on computationally expensive enumeration while lacking a principled mechanism for selecting coordination-effective subsets.

In this paper, we address PAR subset selection by learning an adaptive strategy 
conditioned on the current state. We formulate the problem as a contextual bandit, 
where the PAR policy maps states to subset sizes and receives an uncertainty-weighted 
reward derived from the estimated joint-action value. This encourages both high-value 
coordination and conservative, reliable value estimation. We propose PAR 
Learning-based Conservative Q-Learning (PLCQL), which jointly trains the MARL policy 
and the PAR strategy from offline data. By sampling a single PAR subset per update 
step rather than enumerating all configurations, PLCQL substantially reduces 
computational cost while retaining the benefits of partial action anchoring. To our 
knowledge, this is the first work to formulate PAR subset selection as a learning 
problem in offline MARL.

The main contributions are:
\begin{itemize}
    \item We propose PLCQL, a framework that learns a state-dependent PAR policy via contextual bandits, jointly considering agent coordination and value uncertainty. PLCQL reduces the number of per-iteration Q-function evaluations from 
    $n$ to $1$ compared to SPaCQL while retaining the benefits of partial 
    action anchoring.

    \item We derive a value estimation error bound showing that the estimation 
    error scales linearly with the expected number of deviating agents induced 
    by the learned PAR policy.

    \item We provide an empirical evaluation across MPE, MaMuJoCo, and SMAC 
    benchmarks, demonstrating state-of-the-art performance on 66\% of tasks and 
    outperforming SPaCQL on 84\% of tasks.
\end{itemize}

\section{Related Work}
\textbf{Offline Reinforcement Learning.}
Offline reinforcement learning (RL) aims at learning policies from fixed datasets without additional interaction with the environment. A central challenge in offline RL is distributional shift, where the learned policy may select actions that are not well represented in the dataset, leading to unreliable value estimates. The main idea of addressing this issue is to constrain policy updates or penalize out-of-distribution (OOD) actions. Methods like \cite{Fujimoto2021ALearning,Fujimoto2019,Kumar2019StabilizingReduction,Wu2019} rely on behaviour regularization to constrain policy updates, encouraging learned policies to stay close to the behaviour policy that generated the dataset. Other works like \cite{Kumar2020ConservativeLearning,An2021Uncertainty-BasedQ-Ensemble,Kostrikov2021OfflineRegularization} focus on conservative value estimation, which reduces overestimation of OOD actions by explicitly penalizing Q-values for unseen actions. While these methods have been effective in single-agent settings, extending them to multi-agent environments introduces additional challenges due to the exponential growth of the joint-action space.

\textbf{Offline MARL.}
Offline MARL extends offline RL to environments involving multiple agents. Compared to the single-agent setting, offline MARL faces significantly greater challenges due to the combinatorial explosion of joint actions. Even when individual agents' actions are well represented in the dataset, the dataset may contain only sparse coverage of joint-action combinations. Recent studies have explored adapting offline RL methods to the multi-agent setting by typically extending conservative value estimation or behavior regularization to joint policies \cite{Kostrikov2022OfflineQ-Learning,Jin2025PartialMARLb,Pan2022PlanRectification,Shao2023CounterfactualLearning,Yang2021BelieveLearning,Li2025DOF:LEARNING}. However, these methods often treat the joint action as a single entity and do not explicitly address the sparse coverage of joint-action combinations in offline datasets. Consequently, mitigating distributional shift in the joint-action space remains a major challenge in offline MARL.

\textbf{PAR for Distribution Shift Mitigation.}
To address the joint-action distribution shift problem, recent work proposes PAR strategies \cite{Jin2025PartialMARLb}. PAR mitigates OOD issues by modifying the actions of only a subset of agents while keeping the remaining actions fixed to those observed in the dataset. This design anchors part of the joint action to data-supported actions, thereby reducing the likelihood of evaluating completely unseen joint-action combinations. A representative approach is Soft Partial Conservative Q-Learning (SPaCQL), which evaluates multiple PAR configurations by considering different subset sizes and randomly sampling agent combinations. SPaCQL further introduces an uncertainty-based weighting scheme to combine the resulting target values. While this approach mitigates distributional shift, it requires evaluating multiple subset configurations at each update step, resulting in significant computational overhead. Moreover, the weighting scheme relies on uncertainty estimates and does not explicitly prioritize subsets that may lead to better agent coordination.

\textbf{Bandits for algorithmic adaptation.}
Bandit algorithms have been applied to adaptively select algorithmic components during training, including curriculum schedules~\cite{Graves2017AutomatedNetworks} and hyperparameter configurations~\cite{Li2018Hyperband:Optimization}. These approaches share the intuition that the best algorithmic choice is often context-dependent, and that a learned selector can outperform any fixed configuration. Our use of a contextual bandit for PAR subset selection is in this spirit: rather than fixing  subset size $k$ or enumerating all possibilities as in SPaCQL, the bandit learns a state-dependent selection policy that adjusts conservatism on a per-state basis. Since the choice of $k$ affects only the current Bellman target and does not alter the environment state or the dataset, the bandit formulation captures the full decision structure relevant to PAR subset selection without the overhead of a full MDP over $k$. To our knowledge, PLCQL is the first method to use a contextual bandit to adaptively select PAR subset sizes in offline MARL.

In contrast to existing methods, our work focuses on learning an adaptive strategy for selecting PAR subsets. We formulate subset selection as a contextual bandit problem and propose PLCQL, which learns a PAR policy that dynamically selects agent subsets based on state context and value uncertainty. This approach eliminates the need for enumerating multiple subset configurations and introduces a principled mechanism for identifying coordination-effective subsets, thereby facilitating robust offline multi-agent policy learning under distributional shift.

\section{Preliminaries}
We study offline MARL in a Decentralized Markov Decision Process (Dec-MDP) $\langle \mathcal{S}, \{\mathcal{A}_i\}_{i=1}^n, P, R, n, \gamma \rangle$, where $n$ agents interact in state space $\mathcal{S}$. At each timestep, agent $i$ selects action $a_{t,i} \in \mathcal{A}_i$, forming joint action $\boldsymbol{a}_t \in \boldsymbol{\mathcal{A}} = \times_{i=1}^n \mathcal{A}_i$. The environment transitions to $s_{t+1} \sim P(\cdot|s_t, \boldsymbol{a}_t)$ and yields reward $r_t = R(s_t, \boldsymbol{a}_t)$.

In the offline setting, agents cannot interact with the environment. Learning occurs solely from a static dataset $\mathcal{D} = \{(s, \boldsymbol{a}, r, s')_k\}$ collected by behavior policy $\boldsymbol{\mu}(\boldsymbol{a}|s)$. The goal remains learning a joint policy $\boldsymbol{\pi}(\boldsymbol{a}|s)$ that maximises expected return $J(\boldsymbol{\pi}) = \mathbb{E}[\sum_{t=0}^{\infty} \gamma^t r_t]$, with Q-function defined as:
$$Q^{\boldsymbol{\pi}}(s, \boldsymbol{a}) = \mathbb{E}_{\boldsymbol{\pi}} \left[ \sum_{t=0}^{\infty} \gamma^t R(s_t, \boldsymbol{a}_t) \, \middle| \, s_0=s, \boldsymbol{a}_0=\boldsymbol{a} \right].$$

\textbf{Combinatorial coverage gap.} The key challenge distinguishing offline MARL from offline single-agent RL is combinatorial coverage. In the single-agent case, a dataset of size $|\mathcal{D}|$ may cover a reasonable fraction of the individual action space $\mathcal{A}$. In the multi-agent case, the joint action space $\boldsymbol{\mathcal{A}} = \times_{i=1}^n \mathcal{A}_i$ grows exponentially in $n$. A dataset of the same size therefore covers an exponentially smaller fraction, making OOD joint actions unavoidable even when each agent's marginal action distribution is well represented. This is the fundamental motivation for PAR: by anchoring $n{-}k$ agents to dataset actions, it ensures that at most $k$ dimensions of the joint action are OOD per update step, directly bounding the degree of distributional shift. Formalising how to select this bound adaptively requires a decision-making framework that is responsive to state-level uncertainty --- motivating the use of contextual bandits, described next.

\textbf{Contextual bandits.}
The contextual bandit problem models decision making under context-dependent rewards, in which an agent selects actions based on observed context in order to maximise the expected reward under the context distribution. At each round $t$, the agent observes a context $x_t \in X$, selects an action $a_t \in A$ according to a policy $\pi(a|x)$ and receives a reward $r_t = r(x_t,a_t)$. Unlike reinforcement learning, contextual bandits do not involve state transitions, and the objective is to learn a policy that maximises the expected reward under the context distribution. Formally, the goal is to learn a policy that maximises
$$\mathbb{E}_{x\sim \mathcal{D}, a \sim \pi(\cdot|x)} [r(x,a)].$$
Contextual bandits are widely used to model decision problems where each action only influences the immediate reward.


\section{Methodology}
In this work, we adopt the contextual bandit framework to model PAR subset selection. Specifically, the context corresponds to the bootstrap state $s'$, i.e., the next state in each sampled transition.
 The action corresponds to the subset size $k \in \{1, \cdots, n\}$, where $n$ is the number of agents. Given the selected $k$, a subset $c \subseteq \{1, \cdots, n\}$ with $|c| = k$ is uniformly sampled at each step and used for partial action replacement. The actions of agents in $c$ are updated according to the current policy, while the actions of the remaining agents are fixed to those observed in the dataset. The resulting joint action is then evaluated by the value function. The reward of the contextual bandit is defined based on the estimated joint-action value, weighted by the uncertainty of the value function, which encourages the selection of subsets that yield high-value and reliable policy updates. Under this formulation, learning the PAR strategy reduces to learning a contextual bandit policy that maps states to subset sizes, from which agent subsets are uniformly sampled, in order to maximise the expected reward.

\textbf{Why a contextual bandit rather than a full MDP?}
One may ask whether modelling $k$ selection as a full MDP would be more expressive. We argue the bandit formulation is both sufficient and preferable here. The subset size $k$ affects only the \emph{current} Bellman target; it does not alter the environment state, the dataset, or the trajectory of the MARL policy, so there is no meaningful state transition induced by the $k$ selection itself. Modelling it as an MDP would therefore introduce unnecessary complexity—requiring a separate replay buffer and multi-step return estimation for the meta-policy—without capturing any additional structure. The bandit formulation keeps the framework lightweight and easy to tune, while fully capturing the decision structure relevant to PAR subset selection.

\subsection{PAR Policy}
Instead of enumerating multiple subset configurations, we introduce PAR Learning-based Conservative Q-Learning (PLCQL), which learns an adaptive partial action replacement (PAR) policy. The PAR policy
$$\pi_{par}: \mathcal{S} \rightarrow \{1, \cdots, n\},$$
maps states to PAR strategies that determine how many agents' actions should be replaced during target value computation.

We denote the PAR action as $k \in \{1, \cdots, n\}$, representing the number of agents whose actions will be replaced. The PAR policy is assumed to be stochastic, such that
$$k \sim \pi_{par}(\cdot |s').$$

Given a selected action $k$, the corresponding Bellman operator $\mathcal{T}^k$ replaces exactly $k$ agents' actions in the next-state joint action. Specifically, $k$ agents are chosen uniformly at random to deviate from the behaviour policy, while the remaining $n{-}k$ agents retain the actions recorded in the dataset.

Formally, for any $k \in \{1, \cdots, n\}$, we define the Bellman operator $\mathcal{T}^k$ as
\begin{equation}
    \mathcal{T}^{(k)}Q(s,\boldsymbol{a}) := \mathbb{E}_{s'\sim\mathcal{D},\, \boldsymbol{a}'^{(k)}}\left[ r+\gamma\,Q\left(s',\,\boldsymbol{a}'^{(k)}\right)\right],
\end{equation}
where $\boldsymbol{a}'^{(k)}$ denotes the next-state joint action in which exactly $k$ agents' actions are replaced according to the current policy, while the remaining actions are taken from the logged dataset.
Each $\mathcal{T}^{(k)}$ is a $\gamma$-contraction under the $\ell_\infty$ norm. Since the PAR policy induces a stochastic mixture over these operators, the resulting operator $\mathcal{T}^{par}$ also preserves the $\gamma$-contraction property.

The objective of the PAR policy is to improve offline MARL policy learning by selecting effective subset sizes. To encourage reliable policy updates, we incorporate the uncertainty of value estimation into the reward signal used to train the PAR policy. Specifically, we define an uncertainty-weighted reward
\begin{equation}
r_{par}(s', k) = \bigl(\sigma(-u_Q \cdot T)+0.5\bigr)\, Q_{\theta_1}(s', \boldsymbol{a}'^{(k)}),
\end{equation}
where the uncertainty $u_Q$ is measured using the variance of an ensemble of Q-functions,
$u_Q = \sqrt{\operatorname{Var}_{j}\bigl[Q_{\theta_j}(s', \boldsymbol{a}'^{(k)})\bigr]}$, and $T > 0$ is a temperature parameter controlling the sensitivity of the penalty to ensemble disagreement. High ensemble disagreement indicates poor coverage of the corresponding joint action in the dataset. Therefore, the uncertainty term reduces the reward assigned to PAR actions that rely on unreliable value estimates, encouraging the PAR policy to select subset sizes that balance value improvement and estimation reliability.

\subsection{The PLCQL Algorithm}
We learn the PAR policy using Proximal Policy Optimisation (PPO) \cite{Schulman2017ProximalAlgorithms}, a policy gradient algorithm that provides stable and efficient updates for stochastic policies. In our framework, the PAR policy $\pi_{par}(k|s')$ is parameterised by a neural network that takes the bootstrap state $s'$ as input
and outputs a probability distribution over subset sizes $k \in \{1, \cdots, n\}$.

During training, for each transition sampled from the offline dataset, the PAR policy selects a subset size $k$ according to $\pi_{par}(k|s')$. A subset of $k$ agents is then uniformly sampled, and the corresponding PAR Bellman operator $\mathcal{T}^k$ is used to construct the target value for updating the Q-function. The reward signal for the PAR policy is given by the uncertainty-weighted reward $r_{par}(s', k)$ defined above.
The PPO objective is used to update the PAR policy by maximising the expected reward while constraining policy updates to remain close to the previous policy. Specifically, the policy parameters are optimised using the clipped surrogate objective
\begin{equation}
    \mathbb{E}_{(s,\boldsymbol{a})\sim \mathcal{D}}\!\left[\min \!\left(\rho_{par}\, \hat{A}_{par},\; \operatorname{clip}\!\left(\rho_{par}, 1{-}\epsilon, 1{+}\epsilon\right) \hat{A}_{par} \right)\right],
\end{equation}
where $\rho_{par} = \frac{\pi_{par}(k|s')}{\pi^{old}_{par}(k|s')}$ is the probability ratio between the new and old policies, and $\hat{A}_{par} = r_{par}(s', k)- V_{\theta_{par}}(s')$ denotes the estimated advantage of selecting subset size $k$. The baseline $V_{\theta_{par}}(s') = \mathbb{E}_{k}[r_{par}(s', k)]$ is learned by minimising
\begin{equation}
    \mathcal{L}_v = \mathbb{E}_{s'\sim \mathcal{D},\, k\sim\pi_{par}} \!\left[\bigl(V_{\theta_{par}}(s') - r_{par}(s', k)\bigr)^2 \right].
\end{equation}

The learning objective for PLCQL combines the TD error for our adaptive Bellman operator with a conservative penalty:
\begin{equation}
    \mathcal{L}(\theta) = \mathbb{E}_{\mathcal{D}} \left[ \left(Q_\theta(s, \boldsymbol{a})-Y_{\text{par}}\right)^2 \right] + \xi_c,
\end{equation}
where $\xi_c= \alpha \sum_{i=1}^n \lambda_i \bigl( \mathbb{E}_{(s, \boldsymbol{a}_{-i}) \sim \mathcal{D},\, a_i \sim \pi_i}[Q_\theta(s, a_i, \boldsymbol{a}_{-i})] - \mathbb{E}_{(s,\boldsymbol{a}) \sim \mathcal{D}}[Q_\theta(s, \boldsymbol{a})] \bigr)$ is the conservative penalty adopted from CFCQL~\cite{Shao2023CounterfactualLearning}.
The target $Y_{\text{par}}$ is constructed using the PAR policy, with an ensemble minimum for conservatism:
\begin{equation}
    Y_{\text{par}} := r+\gamma \min_{j}Q^{\text{tar}}_j\bigl(s',\boldsymbol{a}'^{(k)}\bigr), \quad k \sim \pi_{par}(\cdot |s').
\end{equation}

A complete description is given in Algorithm~\ref{alg:PLCQL} and Figure~\ref{PLCQL} illustrates the overall workflow. Through joint training, PLCQL simultaneously learns the MARL value function and the PAR policy, enabling adaptive selection of subset sizes that improve offline policy learning while mitigating distributional shift.

\begin{figure}[t]
  \centering
\includegraphics[width=0.8\columnwidth]{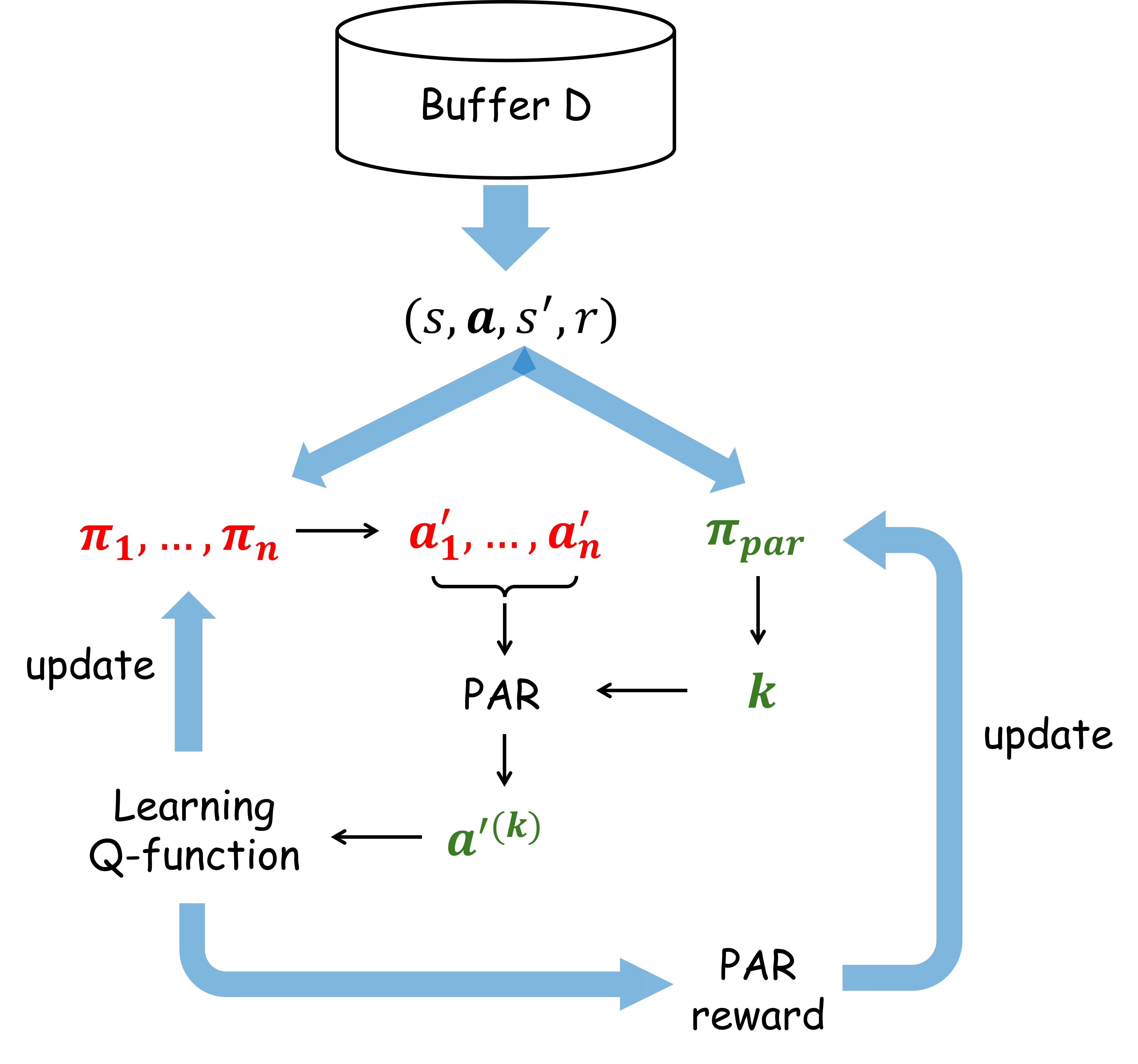}
\caption{Illustration of the PLCQL algorithm.}
\label{PLCQL}
\end{figure}

\begin{algorithm}[t!]
    \caption{PLCQL}
    \label{alg:PLCQL}
    \begin{algorithmic}[1]
    \STATE \textbf{Initialize:} Q-ensemble $\{Q_{\theta_j}\}$, target ensemble $\{\bar Q_{\bar\theta_j}\}$, policies $\{\pi_{\psi_i}\}_{i=1}^n$, target policies $\{\bar\pi_{\bar\psi_i}\}_{i=1}^n$, replay buffer $\mathcal{D}$, PAR policy $\pi_{par}$, value baseline $V_{\theta_{par}}$
    \FOR{each iteration}
        \STATE Sample batch $\mathcal{B}=\{(s,\boldsymbol{a},r,s',\boldsymbol{a}')\}$ from $\mathcal{D}$
        \STATE Set $\mathcal{L}(\theta)\leftarrow0$, $\mathcal{L}_v\leftarrow0$
        \FOR{each transition $(s,\boldsymbol{a},r,s',\boldsymbol{a}')$ in $\mathcal{B}$}
            \STATE Sample $k \sim \pi_{par}(\cdot|s')$
            \STATE Sample $k$ agent indices $\{\sigma_\rho\}_{\rho=1}^k$
            \STATE Sample $\{a_{\sigma_\rho}^{\pi} \sim \pi_{\sigma_\rho}(\cdot|s')\}_{\rho=1}^k$
            \STATE Construct $\boldsymbol{a}'^{(k)} \leftarrow \boldsymbol{a}'$: for each $\sigma_\rho$, replace the $\sigma_\rho$-th component with $a_{\sigma_\rho}^{\pi}$
            \STATE $Y_{\text{par}} = r + \gamma \min_j\bar Q_{\theta_j}(s',\boldsymbol{a}'^{(k)})$
            \STATE $\mathcal{L}(\theta) \mathrel{+}= \sum_{j} \bigl(Q_{\theta_j}(s,\boldsymbol{a})-Y_{\text{par}}\bigr)^2$

            \STATE $u_Q \leftarrow \sqrt{\operatorname{Var}_{j}\bigl[Q_{\theta_j}(s', \boldsymbol{a}'^{(k)})\bigr]}$
            \STATE $w \leftarrow \sigma(-u_Q \cdot T) + 0.5$
            \STATE $r_{par} \leftarrow w\, Q_{\theta_1}(s', \boldsymbol{a}'^{(k)})$ \hfill \textit{// uncertainty-weighted reward at $(s', \boldsymbol{a}'^{(k)})$}
            \STATE $\mathcal{L}_v \mathrel{+}= \bigl(V_{\theta_{par}}(s') - r_{par}\bigr)^2$
        \ENDFOR
     \STATE $\theta\leftarrow\theta-\eta_\theta\nabla_\theta \bigl(\mathcal{L}(\theta)/|\mathcal{B}| + \xi_c\bigr)$

        \STATE $\bar\theta\leftarrow(1-\tau)\bar\theta+\tau \theta$
        \STATE $\theta_{par}\leftarrow\theta_{par}-\eta_{\theta_{par}}\nabla_{\theta_{par}} \mathcal{L}_v$
        \STATE \textcolor{gray}{\textit{--- Agent policy update ---}}
        \FOR{each agent $i$}
            \STATE $\psi_i \leftarrow \psi_i + \eta_\pi\nabla_{\psi_i} \mathbb{E}_{s, \boldsymbol{a}_{-i}\sim \mathcal{D},\, a_i\sim \pi_i} Q_{\theta_1}(s, \boldsymbol{a})$
            \STATE $\bar\psi_i\leftarrow(1-\tau)\bar\psi_i+\tau \psi_i$
        \ENDFOR
        \STATE \textcolor{gray}{\textit{--- PAR policy update (PPO) ---}}
        \STATE $\hat{A}_{par} \leftarrow r_{par} - V_{\theta_{par}}(s')$
        \STATE Update $\pi_{par}$ to maximise:
        $\mathbb{E}_{(s,\boldsymbol{a})\sim \mathcal{D}}\!\left[\min\!\left(\rho_{par}\,\hat{A}_{par},\;\operatorname{clip}(\rho_{par}, 1{-}\epsilon, 1{+}\epsilon)\,\hat{A}_{par}\right)\right]$
    \ENDFOR
    \end{algorithmic}
\end{algorithm}


\subsection{Theoretical Value-Error Bound of PLCQL}

PLCQL retains strong theoretical guarantees without introducing additional structural assumptions on the Q-function. The expected distribution shift under the PLCQL operator scales linearly with the expected number of deviating agents, $\mathbb{E}[k|s] = \sum_k \pi_{par}(k|s) \cdot k$, which leads directly to our value-error bound. Here $\varepsilon_{\mathrm{Subopt}}$ denotes the suboptimality of the learned policy relative to the optimal in-support policy, and $\varepsilon_{\mathrm{FQI}}$ is the fitted Q-iteration approximation error arising from function approximation and finite data.

\begin{theorem}[PLCQL Value-Error Bound]
Let $\hat{Q}$ be the function learned by PLCQL and define the average single-agent policy deviation as $\overline{\mathrm{TV}}(\pi, \mu) = \frac{1}{n} \sum_{i=1}^n \mathrm{TV}(\pi_i, \mu_i)$. The value estimation error is bounded by:
$$
|V^\pi-\hat{V}^\pi| \le \varepsilon_{\mathrm{Subopt}} + \varepsilon_{\mathrm{FQI}} + \frac{4\gamma}{(1-\gamma)^2}\, \mathbb{E}_{s' \sim d^\pi,\, k \sim \pi_{par}(\cdot|s')} \!\left[k\right] \cdot \overline{\mathrm{TV}}(\pi, \mu).
$$
\end{theorem}

\begin{proof}[Proof Sketch]
The proof follows the same structure as the PAR value-error analysis in \cite{Jin2025PartialMARLb}.
 The key difference is that the distribution shift term $W_1(d^\pi, d^\mu)$ is replaced by the expected shift under the PLCQL operator. This expectation is taken over the mixture $\sum_k \pi_{par}(k|s') \pi^{(k)}$, yielding a dependency on the state-dependent expected number of deviations, $\mathbb{E}[k|s']$.
\end{proof}

The error scales with the effective number of deviations $k_{\text{eff}}$, which PLCQL adapts on a per-state basis. When $\pi_{par}$ concentrates on $k{=}1$, only a single agent's action is replaced at each update, minimising distributional shift and yielding the tightest possible value-error bound. As probability mass shifts toward larger $k$, more agents deviate simultaneously, increasing the expected shift and approaching the full-joint error bound. This illustrates how PLCQL adaptively balances conservatism and coordination across different states.


\section{Experimental Settings and Results}

\subsection{Experimental Setup} We conduct a comprehensive evaluation of PLCQL on three widely used offline MARL benchmarks: Multi-Agent Particle Environments (MPE) \cite{Lowe2017Multi-agentEnvironments}, Multi-Agent MuJoCo (MaMuJoCo) \cite{Peng2021FACMAC:Gradients}, and StarCraft Multi-Agent Challenge (SMAC) \cite{Samvelyan2019TheChallenge}.

Following recent work \cite{Jin2025PartialMARLb,Shao2023CounterfactualLearning,Pan2022PlanRectification,Kostrikov2022OfflineQ-Learning}, we use the same offline datasets to ensure fair comparisons. For MPE, there are three distinct tasks: Cooperative Navigation (CN), Predator-Prey (PP), and World. For MaMuJoCo, there is Half-Cheetah (Half-C). The SMAC benchmark includes four maps with varying agent counts and difficulties: 2s3z, 3s\_vs\_5z, 5m\_vs\_6m, and 6h\_vs\_8z.

To evaluate performance under different levels of dataset quality, each MPE and MaMuJoCo task includes four dataset types: Expert (Exp), Medium (Med), Medium-Replay (Med-R), and Random (Rand). For SMAC, we follow the dataset splits from prior works \cite{Shao2023CounterfactualLearning,Jin2025PartialMARLb}, which include Medium (Med), Medium-Replay (Med-R), Expert (Exp), and Mixed datasets.
 These datasets provide varying levels of policy optimality and coverage of the state–action space.

Our implementation uses an ensemble of ten Q-networks to estimate value functions and quantify value uncertainty. The PAR policy is parameterized by a neural network with two hidden layers of 64 units each. All other hyperparameters follow those specified in the original SPaCQL \cite{Jin2025PartialMARLb} implementation to ensure a fair comparison. To improve the reliability of the results, we run all experiments with five random seeds and report the mean and standard deviation of the normalized scores. All experiments are implemented in PyTorch and executed on NVIDIA Tesla V100 GPUs.

\subsection{Baselines}
We compare PLCQL against a suite of state-of-the-art offline MARL algorithms, selected to represent the major methodological directions in the literature:
\begin{itemize}
\item Policy-constrained methods: OMAR \cite{Pan2022PlanRectification}, IQL \cite{Kostrikov2022OfflineQ-Learning}, MA-TD3+BC \cite{Pan2022PlanRectification}, AWAC \cite{Nair2020AWAC:Datasets}, and MAICQ \cite{Yang2021BelieveLearning}.
\item Value-constrained methods: MACQL \cite{Shao2023CounterfactualLearning}, CFCQL \cite{Shao2023CounterfactualLearning}, and SPaCQL \cite{Jin2025PartialMARLb}.
\item Diffusion-based methods: DoF \cite{Li2025DOF:LEARNING}.
\end{itemize}
These baselines provide strong comparisons across different paradigms, enabling a thorough evaluation of the effectiveness of PLCQL.

\subsection{Main Results}
Performance is measured using normalised scores over five random seeds, with mean and standard deviation reported for each task and dataset type. Results are presented in Tables~\ref{tab:main_results} and~\ref{tab:starcraft_results}, where baseline scores are taken from the corresponding papers. Overall, PLCQL achieves the highest normalised scores on approximately 66\% of tasks (21 out of 32).

Compared with policy-constrained methods (OMAR, IQL, MA-TD3+BC), PLCQL demonstrates clear advantages on nearly all tasks, achieving higher joint-action values while maintaining stable learning. Against value-constrained methods (MACQL, CFCQL, and SPaCQL), PLCQL attains comparable or better performance while incurring significantly lower computational cost, owing to its single-step PAR subset update per iteration. When compared with the diffusion-based approach DoF, PLCQL shows competitive or superior performance on most tasks, particularly on non-expert datasets where effective exploration of OOD joint actions combined with accurate value estimation is most critical.

In direct comparison with SPaCQL—which enumerates multiple subset configurations independently of state context—PLCQL achieves better performance on approximately 84\% of tasks (27 out of 32). This highlights the effectiveness of adaptive PAR learning: by dynamically selecting the number of agents to replace based on state and estimated joint-action value, PLCQL better balances policy improvement against conservatism.

\subsection{Ablation Study}
To isolate the contribution of each component of PLCQL, we conduct ablations on the CN (Med-R) and 2s3z (Med) tasks, which span both continuous and discrete action spaces.

\textit{Effect of adaptive PAR learning.}
We compare PLCQL against two fixed-$k$ variants: \textbf{PLCQL-}$k{=}1$ (single-agent replacement only) and \textbf{PLCQL-}$k{=}n$ (full joint replacement, equivalent to standard CQL without anchoring). Results in Table~\ref{tab:ablation} show that no single fixed $k$ consistently dominates across tasks and dataset types, validating the need for state-adaptive selection. The learned PAR policy outperforms fixed-$k$ baselines on all tasks.

\textit{Effect of uncertainty weighting.}
We ablate the uncertainty-weighted reward by setting $w{=}1$ (removing the $\sigma(-u_Q T)$ penalty). Performance degrades consistently, and most sharply on medium-replay datasets where value estimates are least reliable. This confirms that the uncertainty penalty steers the PAR policy away from poorly supported joint actions and is not merely a scaling factor.


\begin{table}[h]
\centering
\caption{Ablation results (normalised score / win rate) on CN Med-R and 2s3z Med. PLCQL (full) uses the learned PAR policy with uncertainty weighting.}
\label{tab:ablation}
\begin{tabular}{lcc}
\toprule
Variant & CN Med-R & 2s3z Med \\ \midrule
PLCQL (full) & \textbf{95.4 $\pm$ 6.1} & \textbf{0.51 $\pm$ 0.19} \\
Fixed $k{=}1$ & 58.2 $\pm$ 8.4 & 0.43 $\pm$ 0.14 \\
Fixed $k{=}n$ & 52.2 $\pm$ 9.6 & 0.40 $\pm$ 0.10 \\
No uncertainty weighting ($w{=}1$) & 90.6 $\pm$ 8.1 & 0.48 $\pm$ 0.23 \\
\bottomrule
\end{tabular}
\end{table}

\subsection{Computational Cost}
PLCQL evaluates a single PAR subset per transition update. SPaCQL, by contrast, evaluates $n$ subset configurations per update step. On the 6-agent SMAC maps ($n{=}6$), this corresponds to 6 configurations for SPaCQL versus 1 for PLCQL. In practice, we observe a wall-clock speedup 
per training iteration on V100 hardware, confirming that adaptive PAR learning substantially reduces computational overhead without sacrificing policy quality.




\section{Discussion and Conclusions}
In this work, we introduced PLCQL, a novel framework for offline MARL that leverages PAR to mitigate distributional shift in the combinatorial joint-action space. By formulating the subset selection problem as a contextual bandit, PLCQL learns a state-dependent PAR policy that adaptively determines the number of agents to replace at each update step, balancing coordination improvement with conservative value estimation.

Our theoretical analysis shows that the expected distribution shift under the PAR operator scales linearly with the expected number of deviating agents, yielding a clear and interpretable value-error bound. Empirically, PLCQL achieves strong performance across MPE, MaMuJoCo, and SMAC, outperforming state-of-the-art baselines on approximately 66\% of tasks while reducing the computational cost of PAR-based methods by eliminating the need to enumerate multiple subset configurations. The uncertainty-weighted reward mechanism effectively prioritises subsets that are both high-value and well-supported by the dataset, contributing to more reliable policy learning.

A central contribution of our work is the introduction of adaptive PAR learning. By conditioning subset size selection on the environment state, PLCQL provides an efficient and principled mechanism for handling distributional shift in offline MARL. This not only improves learning performance but substantially reduces the computational burden compared to methods such as SPaCQL, making PLCQL a practical solution for large multi-agent systems.

Despite these advances, several limitations remain. The current PAR policy selects the \emph{number} of deviating agents, with the specific agents sampled uniformly within that subset; future work could explore more expressive selection mechanisms that explicitly account for inter-agent dependencies. Additionally, while PLCQL reduces per-step computational overhead, reliance on Q-function ensembles may limit scalability to environments with very large numbers of agents.

In conclusion, PLCQL provides an adaptive and theoretically grounded framework for offline multi-agent policy learning. By addressing the challenge of joint-action distributional shift through integrated PAR policy learning and conservative Q-learning, PLCQL advances both the methodology and understanding of offline MARL. We believe that adaptive PAR learning represents a promising avenue for future research in scalable, coordination-intensive multi-agent systems.

\clearpage
\begin{sidewaystable}[!t]
\centering
\caption{The average normalized score on offline MARL tasks. The best performance is highlighted in bold. }
\label{tab:main_results}
\begin{tabular}{c|cccccccccc}
\toprule
 & & Dataset & OMAR & MACQL & IQL & MA-TD3+BC & DoF & CFCQL & SPaCQL & PLCQL \\ [5pt] \hline
\multirow{12}{*}{\rotatebox[origin=c]{90}{MPE}}   & \multirow{4}{*}{CN} & Exp  & 114.9 $\pm$ 2.6   & 12.2$\pm$31  & 103.7$\pm$2.5 & 108.3$\pm$3.3 &\textbf{136.4$\pm$3.9} &  112$\pm$4 &111.9$\pm$4.5 &110.9$\pm$3.0   \\ [5pt]
                          &                                         & Med  & 47.9$\pm$18.9  & 14.3$\pm$20.2  & 28.2$\pm$3.9  & 29.3$\pm$4.8 & 75.6$\pm$8.7  & 65.0$\pm$10.2 & 78.6$\pm$6.4 &{ {\textbf{85.1$\pm$5.9}}} \\ [5pt]
                          &                                           & Med-R   & 37.9$\pm$12.3  & 25.5$\pm$5.9  & 10.8$\pm$4.5 & 15.4$\pm$5.6 & 57.4$\pm$6.8& 52.2$\pm$9.6 &71.9$\pm$13.2 &{ {\textbf{95.4$\pm$6.1}}} \\ [5pt]
                          &                                           & Rand  & 34.4$\pm$5.3  & 45.6$\pm$8.7  & 5.5$\pm$1.1 & 9.8$\pm$4.9 & 35.9$\pm$6.8 & 62.2$\pm$8.1 & 78.2$\pm$14 &{ {\textbf{88.3$\pm$4.8}}} \\ [5pt] \cline{2-11} 
                          
                          & \multirow{4}{*}{PP}            & Exp  & 116.2$\pm$19.8  & 108.4$\pm$21.5  & 109.3$\pm$10.1 & 115.2$\pm$12.5  &\textbf{125.6$\pm$8.6} & 118.2$\pm$13.1 &111.2$\pm$16.4 &107.1$\pm$10.8 \\  [5pt]
                          &                                           & Med  & 66.7$\pm$23.2  & 55$\pm$43.2  & 53.6$\pm$19.9 & 65.1$\pm$29.5  & \textbf{86.3$\pm$10.6} & 68.5$\pm$21.8 &61.9$\pm$20 &84.0$\pm$29.0 \\  [5pt]
                          &                                           & Med-R   & 47.1$\pm$15.3  & 11.9$\pm$9.2  & 23.2$\pm$12 & 28.7$\pm$20.9  & 65.4$\pm$12.5 & 71.1$\pm$6  & 75.0$\pm$12.7 &{ {\textbf{87.3$\pm$10.4}}} \\  [5pt]
                          &                                           & Rand  & 11.1$\pm$2.8  & 25.2$\pm$11.5  & 1.3$\pm$1.6 & 5.7$\pm$3.5  &16.5$\pm$6.3 & 78.5$\pm$15.6 &89.4$\pm$13.7 &\textbf{92.8$\pm$6.0} \\  [5pt] \cline{2-11} 
                          & \multirow{4}{*}{World}                    & Exp  & 110.4$\pm$25.7  & 99.7$\pm$31  & 107.8$\pm$17.7 & 110.3$\pm$21.3  & \textbf{135.2$\pm$19.1} & 119.7$\pm$26.4 &112.3$\pm$7.8 &133.0$\pm$9.4 \\ [5pt]
                          &                                           & Med  & 74.6$\pm$11.5  & 67.4$\pm$48.4  & 70.5$\pm$15.3 & 73.4$\pm$9.3 & 85.2$\pm$11.2 & 93.8$\pm$31.8 &98.1$\pm$17.7 &{ {\textbf{104.9$\pm$22.0}}} \\ [5pt]
                          &                                           & Med-R   & 42.9$\pm$19.5  & 13.2$\pm$16.2  & 41.5$\pm$9.5 & 17.4$\pm$8.1 & 58.6$\pm$10.4 & 73.4$\pm$23.2 & 105.2$\pm$11.1 &{ {\textbf{107.5$\pm$11.3}}} \\ [5pt]
                          &                                           & Rand  & 5.9$\pm$5.2  & 11.7$\pm$11  & 2.9$\pm$4.0 & 2.8$\pm$5.5 & 13.1$\pm$2.1 & 68$\pm$20.8 & \textbf{94.3$\pm$7.4} & 84.9$\pm$6.9 \\  [5pt]\hline
\multirow{4}{*}{\rotatebox[origin=c]{90}{MaMujoco}} & \multirow{4}{*}{Half-C}              & Exp  & 113.5$\pm$4.3  & 50.1$\pm$20.1  & 115.6$\pm$4.2  & 114.4$\pm$3.8 & - & \textbf{118.5$\pm$4.9} & 110.5$\pm$5.9 & 111.9$\pm$4.4 \\  [5pt]
                          &                                           & Med  & 80.4$\pm$10.2  & 51.5$\pm$26.7  & \textbf{81.3$\pm$3.7}  & 75.5$\pm$3.7 & -& 80.5$\pm$9.6 &70.3$\pm$7.8 &71.9$\pm$9.1 \\ [5pt]
                          &                                           & Med-R   & 57.7$\pm$5.1  & 37.0$\pm$7.1  & 58.8$\pm$6.8  & 27.1$\pm$5.5 & -& 59.5$\pm$8.2 & 66.1$\pm$3.4 &{ {\textbf{73.1$\pm$5.7 }}}\\ [5pt]
                          &                                           & Rand  & 13.5$\pm$7.0  & 5.3$\pm$0.5  & 7.4$\pm$0.0  & 7.4$\pm$0.0 & - & 39.7$\pm$4.0 & \textbf{43.8$\pm$4.9}  & 40.9$\pm$2.8 \\  [5pt]
                          \bottomrule
\end{tabular}
\end{sidewaystable}
\clearpage
\clearpage
\begin{sidewaystable}
\centering
\caption{Averaged test winning rate in StarCraft II micromanagement tasks. The best performance is highlighted in bold. }
\label{tab:starcraft_results}
\begin{tabular}{cccccccccccc}
\hline 
Map & Dataset & CFCQL & MACQL & MAICQ & OMAR & MADTKD & BC & IQL & AWAC & SPaCQL & PLCQL\\ [5pt]

\hline

\multirow{4}{*}{2s3z}
& Med          & 0.40$\pm$0.10 & 0.17$\pm$0.08 & 0.18$\pm$0.02 & 0.15$\pm$0.04 & 0.18$\pm$0.03 & 0.16$\pm$0.07 & 0.16$\pm$0.04 & 0.19$\pm$0.05 & 0.46$\pm$0.2 & \textbf{0.51$\pm$0.19} \\ [5pt]
& Med-R  & 0.55$\pm$0.07 & 0.12$\pm$0.08 & 0.41$\pm$0.06 & 0.24$\pm$0.09 & 0.36$\pm$0.07 & 0.33$\pm$0.04 & 0.33$\pm$0.06 & 0.39$\pm$0.05 & 0.56$\pm$0.2 & \textbf{0.76$\pm$0.14} \\ [5pt]
& Exp          & 0.99$\pm$0.01 & 0.58$\pm$0.34 & 0.93$\pm$0.04 & 0.95$\pm$0.04 & 0.99$\pm$0.02 & 0.97$\pm$0.02 & 0.98$\pm$0.03 & 0.97$\pm$0.03 & 0.99$\pm$0.01 & \textbf{1 $\pm$ 0}\\ [5pt]
& Mixed           & 0.84$\pm$0.09 & 0.67$\pm$0.17 & 0.85$\pm$0.07 & 0.60$\pm$0.04 & 0.47$\pm$0.08 & 0.44$\pm$0.06 & 0.19$\pm$0.04 & 0.14$\pm$0.04 & 0.96$\pm$0.1 & \textbf{0.97$\pm$0.05}\\ [5pt]
\hline

\multirow{4}{*}{3s\_vs\_5z}
& Med          & 0.28$\pm$0.03 & 0.09$\pm$0.06 & 0.03$\pm$0.01 & 0.00$\pm$0.00 & 0.01$\pm$0.01 & 0.08$\pm$0.02 & 0.20$\pm$0.05 & 0.19$\pm$0.03 & 0.2$\pm$0.4 & \textbf{0.37$\pm$0.41}\\ [5pt]
& Med-R  & 0.12$\pm$0.04 & 0.01$\pm$0.01 & 0.01$\pm$0.02 & 0.00$\pm$0.00 & 0.01$\pm$0.01 & 0.01$\pm$0.01 & 0.04$\pm$0.04 & 0.08$\pm$0.05 & 0.12$\pm$0.24 &\textbf{0.13$\pm$0.18} \\ [5pt]
& Exp          & \textbf{0.99$\pm$0.01} & 0.92$\pm$0.05 & 0.91$\pm$0.04 & 0.64$\pm$0.08 & 0.67$\pm$0.08 & 0.98$\pm$0.02 & 0.99$\pm$0.01 & 0.99$\pm$0.02 & 0.98$\pm$0.06 & 0.98$\pm$0.05\\ [5pt]
& Mixed           & \textbf{0.60$\pm$0.14} & 0.17$\pm$0.10 & 0.10$\pm$0.04 & 0.00$\pm$0.00 & 0.14$\pm$0.08 & 0.21$\pm$0.04 & 0.20$\pm$0.06 & 0.14$\pm$0.03 & 0.43$\pm$0.44 & 0.47$\pm$0.49\\ [5pt]
\hline

\multirow{4}{*}{5m\_vs\_6m}
& Med          & 0.29$\pm$0.05 & 0.01$\pm$0.01 & 0.26$\pm$0.03 & 0.19$\pm$0.06 & 0.21$\pm$0.03 & 0.28$\pm$0.37 & 0.18$\pm$0.02 & 0.22$\pm$0.04 & 0.33$\pm$0.03 & \textbf{0.46$\pm$0.26}\\ [5pt]
& Med-R  & 0.22$\pm$0.06 & 0.01$\pm$0.01 & 0.18$\pm$0.04 & 0.03$\pm$0.03 & 0.16$\pm$0.04 & 0.13$\pm$0.06 & 0.18$\pm$0.04 & 0.18$\pm$0.04 & 0.23$\pm$0.17 & \textbf{0.24$\pm$0.13}\\ [5pt]
& Exp          & \textbf{0.84$\pm$0.03} & 0.01$\pm$0.01 & 0.72$\pm$0.05 & 0.33$\pm$0.06 & 0.58$\pm$0.07 & 0.82$\pm$0.04 & 0.77$\pm$0.03 & 0.75$\pm$0.02 & 0.82$\pm$0.14 & 0.82$\pm$0.23\\ [5pt]
& Mixed           & 0.76$\pm$0.07 & 0.01$\pm$0.01 & 0.67$\pm$0.08 & 0.10$\pm$0.10 & 0.21$\pm$0.05 & 0.21$\pm$0.12 & 0.76$\pm$0.06 & 0.78$\pm$0.02 & 0.78$\pm$0.16 & \textbf{0.82$\pm$0.22}\\ [5pt]
\hline

\multirow{4}{*}{6h\_vs\_8z}
& Med          &0.41$\pm$0.04 & 0.06$\pm$0.04 & 0.19$\pm$0.04 & 0.04$\pm$0.03 & 0.22$\pm$0.07 & 0.40$\pm$0.04 & 0.40$\pm$0.05 & 0.43$\pm$0.06 & 0.42$\pm$0.28 & \textbf{0.48$\pm$0.18}\\ [5pt]
& Med-R  & 0.21$\pm$0.05 & 0.02$\pm$0.04 & 0.07$\pm$0.04 & 0.00$\pm$0.00 & 0.14$\pm$0.04 & 0.11$\pm$0.04 & 0.17$\pm$0.03 & 0.14$\pm$0.04 & 0.22$\pm$0.14 &\textbf{0.24$\pm$0.2}\\ [5pt]
& Exp          & 0.70$\pm$0.06 & 0.02$\pm$0.00 & 0.24$\pm$0.08 & 0.07$\pm$0.03 & 0.22$\pm$0.03 & 0.60$\pm$0.04 & 0.67$\pm$0.03 & 0.60$\pm$0.03 & 0.73$\pm$0.11 & \textbf{0.77$\pm$0.10}\\ [5pt]
& Mixed           & 0.49$\pm$0.08 & 0.01$\pm$0.01 & 0.05$\pm$0.03 & 0.00$\pm$0.00 & 0.25$\pm$0.07 & 0.27$\pm$0.06 & 0.36$\pm$0.05 & 0.35$\pm$0.06 & 0.52$\pm$0.22 &\textbf{0.56$\pm$0.28}\\ [5pt]
\hline
\end{tabular}
\end{sidewaystable}
\clearpage

\newpage



\paragraph{Generative AI disclosure.} The authors used large language model tools to assist with proofreading and \LaTeX{} formatting. All scientific content, experimental design, results, and claims are the sole responsibility of the authors.

\bibliographystyle{splncs04}
\bibliography{main}

\end{document}